\documentclass[11pt]{article}

\usepackage[margin=1in]{geometry}
\usepackage{setspace}
\usepackage{microtype}

\usepackage{graphicx}
\usepackage{subcaption}
\usepackage{booktabs}

\usepackage{caption}
\usepackage[hidelinks]{hyperref}

\usepackage{amsmath}
\usepackage{amssymb}
\usepackage{mathtools}
\usepackage{amsthm}

\usepackage[capitalize,noabbrev]{cleveref}

\theoremstyle{plain}
\newtheorem{theorem}{Theorem}[section]

\theoremstyle{definition}
\newtheorem{definition}[theorem]{Definition}
\newtheorem{assumption}[theorem]{Assumption}
\theoremstyle{remark}

\usepackage[textsize=tiny]{todonotes}

\usepackage{algorithm}
\usepackage{algorithmic}
\AtBeginEnvironment{algorithm}{\singlespacing}
\AtEndEnvironment{algorithm}{\setstretch{2.0}}

\usepackage{enumitem}
\setlist[itemize]{leftmargin=*}
\setlist[enumerate]{leftmargin=*}

\usepackage{natbib}
\def\spacingset#1{\renewcommand{\baselinestretch}%
	{#1}\small\normalsize} \spacingset{1}
	
\usepackage{xr}
\makeatletter

\newcommand*{\addFileDependency}[1]{% argument=file name and extension
	\typeout{(#1)}% latexmk will find this if $recorder=0
	\@addtofilelist{#1}
	\IfFileExists{#1}{}{\typeout{No file #1.}}
}\makeatother

\newcommand*{\myexternaldocument}[1]{%
	\externaldocument{#1}%
	\addFileDependency{#1.tex}%
	\addFileDependency{#1.aux}%
}
\myexternaldocument{0-supple}

\usepackage{xcolor}

\title{\textbf{One-Step Bellman Alignment Enables \\ Provably Efficient Transfer in Online RL}}
\author{
	Elynn Chen$^\sharp$\thanks{Correspondence to E.~Chen at \url{elynn.chen@stern.nyu.edu} and Y.~Yan at \url{yujun.yan@dartmouth.edu}}  \hspace{2ex}
	Enpei Zhang$^\natural$ \hspace{2ex}
	Jinhang Chai$^\diamond$ \hspace{2ex}
	Yujun Yan$^{\dag*}$ \\
	\\ \normalsize
	$^{\sharp}$Department of Technology, Operations, \& Statistics, New York University \\
	$^\diamond$Department of Operations Research \& Financial Engineering, Princeton University \\
	$^{\natural\dag}$ Department of Computer Science, Dartmouth College
}
\date{\today}

\begin{document}

% Title page is usually single-spaced (optional, but common in TRs)
\begingroup
\singlespacing
\maketitle
\endgroup

% Optional: abstract block (single-spaced is common; remove if your 1-all.tex already includes it)
%\begin{abstract}
%\singlespacing
%<Paste abstract here or keep it inside 1-all.tex>
%\end{abstract}

% Main content (double-spaced due to \setstretch{2.0} above)
%!TEX root = 0-main.tex

\begin{abstract}
\spacingset{1.28}
We study online transfer reinforcement learning (RL) in episodic Markov decision processes, where experience from related source tasks is available during learning on a target task. A fundamental difficulty is that task similarity is typically defined in terms of rewards or transitions, whereas online RL algorithms operate on Bellman regression targets. As a result, naively reusing source Bellman updates introduces systematic bias and invalidates regret guarantees.

We identify one-step Bellman alignment as the correct abstraction for transfer in online RL and propose re-weighted targeting (RWT), an operator-level correction that retargets continuation values and compensates for transition mismatch via a change of measure. RWT reduces task mismatch to a fixed one-step correction and enables statistically sound reuse of source data.

This alignment yields a two-stage RWT $Q$-learning framework that separates variance reduction from bias correction. Under RKHS function approximation, we establish regret bounds that scale with the complexity of the task shift rather than the target MDP. Empirical results in both tabular and neural network settings demonstrate consistent improvements over single-task learning and na\"{i}ve pooling, highlighting Bellman alignment as a model-agnostic transfer principle for online RL.
\end{abstract}


\newpage
\spacingset{1.9}
%%%%%%%%%%%%%%%%%%%%%%%%%%%%%%%%%%%%%%%%%%%%%%%%%%%%%%%%%
\section{Introduction}

Reinforcement learning agents are increasingly deployed in settings where experience from related tasks is available, such as simulation-to-real transfer, personalization across users, and decision-making across markets. 
Leveraging such source experience to accelerate learning on a new target task is natural in practice, yet providing principled guarantees for {\it online} transfer under exploration remains challenging.
While many transfer heuristics perform well empirically, existing theory often treats the target task in isolation or relies on assumptions misaligned with how online RL algorithms actually learn \cite{taylor2009transfer,brunskill2013sample}.

A central obstacle is that {\it task similarity is rarely defined at the level where online reinforcement learning operates}. Modern RL algorithms learn by regressing {\it one-step Bellman targets}, which depend jointly on the continuation value and the transition distribution (e.g. \citet{yang2020function,jin2020provably}. 
As a result, similarity assumptions stated in terms of rewards, transitions, or value functions do not directly translate into reusable Bellman samples. Na\"{i}vely pooling source Bellman updates therefore introduces systematic bias even when tasks are intrinsically close, breaking optimism-based exploration and invalidating regret guarantees \cite{srinivas2009gaussian}. 

We show that this failure is structural rather than statistical. The object governing transfer in online RL is the {\it one-step Bellman mismatch} between tasks. Under the standard Bellman operator, this mismatch depends intrinsically on the continuation value and thus varies across dynamic programming iterations, making it impossible to represent as a fixed state-action correction. Consequently, without explicit alignment, there is no principled way to reuse source Bellman samples in online reinforcement learning.

To resolve this obstruction, we introduce {\it re-weighted targeting (RWT)}, an operator-level Bellman alignment that retargets continuation values to the target task and corrects transition mismatch via a change of measure. This alignment removes continuation-value dependence and reduces task mismatch to a {\it fixed one-step correction}. As a result, source Bellman samples become Bellman-consistent for the target task up to a simple residual, making transfer statistically well-posed.

Bellman alignment has direct algorithmic and theoretical consequences. Once mismatch is reduced to a fixed one-step correction, value estimation naturally decomposes into two stages: a {\it source-based baseline} that leverages aligned Bellman targets for variance reduction, and a {\it target-based correction} that learns the structured task shift. Under RKHS function approximation, we establish regret bounds that scale with the complexity of the task shift rather than the target MDP, yielding strict improvements over single-task learning when the shift is structured. Empirically, we observe consistent improvements over both single-task learning and na\"{i}ve pooling across tabular and neural network $Q$-learning, highlighting Bellman alignment as a principled, model-agnostic transfer mechanism.

\textbf{Contributions.}
Our main contributions are:

(i) \emph{Conceptual:} We identify \emph{one-step Bellman alignment} as the correct abstraction for \emph{online} transfer and show that naive Bellman reuse is structurally biased;

(ii) \emph{Algorithmic:} We propose \emph{re-weighted targeting (RWT)}, an operator-level alignment that enables Bellman-consistent reuse of source data and yields a two-stage Q-learning framework for online learning with exploration.

(iii) \emph{Theoretical:} We establish regret bounds under RKHS function approximation that scale with the complexity of the task shift rather than the target MDP;

(iv) \emph{Empirical:} We demonstrate consistent gains over single-task learning and naïve pooling in both tabular and neural network settings.

\subsection{Related Work and Our Distinction}

\textbf{Meta and multitask RL.}
A large literature studies how experience from related tasks can improve sample efficiency in RL, including multitask and meta-RL approaches that learn shared representations, priors, or initializations \citep{taylor2009transfer,brunskill2013sample,calandriello2014sparse,zhou2025prior}. 
Recent theory establishes benefits of representational transfer under shared structure, such as linear or low-rank representations and shared latent dynamics, typically by reducing the effective dimension of the target problem \citep{hu2021near,agarwal2023representational,sam2024limits,Lee2025MultiTaskRL}. 
These works focus on representation sharing across tasks. In contrast, we identify a distinct obstruction specific to \emph{online} value-based RL: transfer must be compatible with Bellman regression targets under exploration, which is not addressed by representation-level similarity alone.

\textbf{Transfer under task shift.}
A more closely related line studies transfer under structured reward/transition differences, often in finite-horizon settings. 
For example, \citet{chen2025transfer,zhang2025transfer,chen2026transfer} analyze transfer $Q$-learning across related MDPs, and \citet{chen2025datadriven} develop data-driven transfer in \emph{batch} $Q$-learning. 
Recent extensions consider composite MDP structures and transition transfer \citep{chai2025transition}. 
However, existing treatments either operate in offline/batch regimes or rely on forms of reuse that do not explicitly resolve the \emph{continuation-value dependence} of Bellman mismatch in online learning. 
Our work differs by formalizing transfer at the \emph{Bellman-operator level} and providing regret guarantees for online learning under exploration after alignment.

\textbf{Kernelized RL and optimism.}
Our analysis builds on kernelized RL with optimism-based exploration, which yields regret bounds scaling with RKHS complexity measures such as information gain rather than the size of the state--action space \citep{yang2020function,vakili2023kernelized,vakili2024kernel}. 
Prior kernel RL work treats the target task in isolation. 
In contrast, we decompose value estimation into a source-based baseline and a residual correction with distinct RKHS complexity, and design upper confidence bounds that jointly capture uncertainty from both components.

\section{Transferability via Bellman Alignment}
\label{sec:RWT-Aligned-Bellman}

A central difficulty in transfer reinforcement learning is that \emph{task similarity is typically defined at a level misaligned with how online RL algorithms learn}.
Although tasks may be similar in rewards or dynamics, episodic value-based methods learn by regressing one-step Bellman targets, which depend jointly on the continuation value and the transition distribution. 
We show that naive Bellman transfer fails due to continuation-value dependence, introduce re-weighted targeting (RWT) to align Bellman operators at one step, and show that this alignment reduces task mismatch to a fixed, learnable one-step correction.

\subsection{Problem Setup}

We consider episodic Markov decision processes (MDPs) with finite horizon $H$ and discount factor $\gamma \in (0,1]$.
There is a target task, indexed by $m=0$, and a collection of source tasks, indexed by $m \in [M]$.
Each task $m$ is specified by an MDP:
$\mathcal M^{(m)} = \bigl(\mathcal S,\mathcal A,H,\{P_h^{(m)}\}_{h=1}^H,\{R_h^{(m)}\}_{h=1}^H,\gamma\bigr)$ with the same state and action spaces $(\mathcal S,\mathcal A)$ but may differ in reward functions $R_h^{(m)}$ and transition dynamics $P_h^{(m)}$.

The learner interacts online only with the target task $\mathcal M^{(0)}$, collecting one trajectory per episode and incurring regret relative to the optimal target policy.

Source tasks may provide offline datasets or auxiliary interaction (e.g., simulators), collected under arbitrary, possibly non-optimistic policies. Our goal is to leverage such source data to accelerate online learning on the target task without compromising exploration or regret guarantees.

\subsection{Why Na\"{i}ve Bellman Transfer Fails}
\label{sec:naive-Bellman-fails}

A natural approach to transfer is to reuse Bellman updates from source tasks when learning the target task. If source and target MDPs are ``close,'' one might expect that pooling Bellman samples reduces variance and accelerates learning. We show that this intuition is incorrect in online RL.

Let $V^{(m)}_{h+1}(s)=\max_{a'}Q^{(m)}_{h+1}(s,a')$ where $Q^{(m)}_{h+1}(s,a')$ is the action-value function at stage $h{+}1$ for task $m$.
The standard Bellman operator for task $m$ at stage $h$ is
\begin{equation*}
(\mathcal B_h^{(m)}V^{(m)}_{h+1})(s,a)
=
R_h^{(m)}(s,a)
+
\gamma\mathbb E_{s'\sim P_h^{(m)}(\cdot\mid s,a)}
[V^{(m)}_{h+1}(s')].
\end{equation*}
There are two distinct sources of misalignment between $\mathcal B_h^{(m)}$ and the target Bellman operator $\mathcal B_h^{(0)}$.
First, the continuation value $V^{(m)}_{h+1}$ is task-dependent.
Even if transitions were identical, replacing $V^{(0)}_{h+1}$ with $V^{(m)}_{h+1}$ introduces bias unless the value functions coincide.
Second, expectations are taken under the wrong transition distribution $P_h^{(m)}$ instead of $P_h^{(0)}$.

Under the standard Bellman operator, the discrepancy between source and target Bellman backups, 
$(\mathcal B_h^{(m)}V^{(m)}_{h+1})(s,a) - (\mathcal B_h^{(0)}V^{(0)}_{h+1})(s,a)$,
depends intrinsically on the continuation values through both the future-value term and the transition distribution. 
Without imposing strong assumptions on direct similarity between value functions -- which are generally untestable in practice -- pooling source Bellman targets therefore produces persistent bias that does not vanish with additional target data. 
This bias is structural, arising from Bellman-operator misalignment, and is incompatible with optimism-based exploration and regret guarantees in online reinforcement learning.

\subsection{Re-Weighted Targeting for Bellman Alignment}
\label{sec:rwt}

We now introduce \emph{re-weighted targeting (RWT)}, an operator-level correction that removes the continuation-value dependence identified above.

Fix stage $h$ and source task $m$.
Assume that the target transition kernel $P_h^{(0)}$ is absolutely continuous with respect to $P_h^{(m)}$, and define the density ratio
\[
\omega_h^{(m)}(s'\mid s,a)
:=
{p_h^{(0)}(s'\mid s,a)}/{p_h^{(m)}(s'\mid s,a)}.
\]
The {\it RWT-aligned Bellman operator} is defined as 
\begin{equation} 
\label{eqn:RWT-aligned-Bellman-operator}
\begin{aligned}
\bigl(\mathcal B_h^{(m\to 0)} V^{(m)}_{h+1}\bigr)(s,a) :=
R_h^{(m)}(s,a)
+
\gamma\,
\mathbb E_{s'\sim P_h^{(m)}(\cdot\mid s,a)}
\!\left[
\omega_h^{(m)}(s'\mid s,a)\, V^{(0)}_{h+1}(s')
\right].
\end{aligned}
\end{equation}
Crucially, the continuation value $V_{h+1}$ is evaluated under the target task, while the density ratio corrects the transition mismatch.

For any bounded continuation value $V_{h+1}$, define the {\it RWT-aligned Bellman difference}
\[
\Delta_h^{(m)}(s,a)
\;:=\;
\bigl(\mathcal B_h^{(0)} V^{(0)}_{h+1}\bigr)(s,a)
-
\bigl(\mathcal B_h^{(m\to 0)} V^{(m)}_{h+1}\bigr)(s,a),
\]
A direct calculation shows that
\begin{equation}
\label{eqn:aligned-bellman-difference}
\Delta_h^{(m)}(s,a)
\equiv
R_h^{(0)}(s,a)-R_h^{(m)}(s,a),
\end{equation}
which depends only on $(s,a)$ and is independent of $V_{h+1}$.
Thus, re-weighted targeting transforms the task-to-task Bellman difference from a continuation-value-dependent object into a fixed one-step correction that is invariant across dynamic programming iterations.

\subsection{Transferability via Structured Reward Differences}
\label{sec:transferability}

After Bellman alignment, task mismatch is fully characterized by the one-step reward difference
\begin{equation}
\label{eqn:reward-diff}
\Delta_{r,h}^{(m)}(s,a)
:=
R_h^{(0)}(s,a)-R_h^{(m)}(s,a).
\end{equation}
We therefore formalize transferability by imposing structure directly on this reward-level discrepancy.

\begin{assumption}[\bf Structured Reward Difference]
\label{assume:reward-difference}
For each stage $h$ and source task $m$, the reward difference $\Delta_{r,h}^{(m)}$ belongs to a function class $\mathcal F_\Delta$ whose complexity is strictly smaller than that of the class $\mathcal F_R$ used to model the target Bellman backup. 
\end{assumption}
This assumption applies exclusively to the {\it one-step} reward difference, which is intrinsic to the task definition and independent of policies or value functions. We do not assume similarity of value functions across tasks, nor do we impose structure on transition kernels or multi-step value differences, which may remain arbitrarily complex. The assumption is meaningful precisely because RWT renders the Bellman mismatch invariant across iterations: learning $\Delta_{r,h}^{(m)}$ suffices to correct aligned source Bellman targets.

Under Assumption \ref{assume:reward-difference}, source data can be used to construct Bellman-consistent pseudo-labels via re-weighted targeting, while target data is required only to estimate the simpler reward correction. This separation, i.e. variance reduction from sources and bias correction from targets, forms the basis of the algorithmic design in Section \ref{sec:OFU-RWT-Q-general} and enables regret bounds that depend on the complexity of the task shift rather than the full target problem.

\section{RWT $Q$-Learning: Algorithmic Framework}\label{sec:OFU-RWT-Q-general}

This section presents a {\it model-agnostic algorithmic framework for online transfer reinforcement learning} based on the Bellman alignment developed in Section \ref{sec:RWT-Aligned-Bellman}. We refer to the resulting method as {\bf RWT $Q$-learning}. Algorithm~\ref{alg:ofu-rwt} gives an optimism-based instantiation, termed {\bf OFU-RWT $Q$-learning}. While OFU is used for theoretical analysis, the Bellman alignment mechanism is independent of the exploration strategy and can be combined with alternatives such as $\varepsilon$-greedy, as used in our empirical evaluation.

Learning proceeds episodically: in each episode, the agent rolls out a policy on the target task, then performs Bellman-aligned updates that combine a source-based baseline for variance reduction with a target-based correction for the residual task shift, before deploying the updated policy in the next episode.

\textbf{Notation.}
We denote by $N$ the total number of episodes of interaction with the target task, each of horizon $H$ (so $T=NH$ total steps). We further denote by $\kappa$ the source-to-target sampling ratio, meaning that by episode $n$ the $m$-th source task provides $n^{(m)}=\kappa^{(m)}n$ samples, and we write $\kappa=\sum_{m=1}^M \kappa^{(m)}$.

\subsection{Synchronous Bellman Updates with RWT}
\label{sec:transfer-RWT-Bellman-update}

After episode $n$, the learner performs {\it synchronous backward updates} of the stage-wise $Q$-functions $\{Q_{n+1,h}\}_{h=1}^H$, starting from $h=H$ down to $1$. 
Let the continuation value be $V_{n,h+1}(s) := \max_{a\in\mathcal A} Q_{n,h+1}(s,a)$.
With transfer, Bellman updates are constructed using RWT, which enables principled reuse of source data while preserving Bellman correctness for the target task.

\textbf{Stage I. (Source-based Bellman alignment).}
For each source task $m$, given the density-ratio estimators $\widehat\omega_h^{(m)}$, we construct {\it RWT-aligned Bellman pseudo-labels}
\begin{equation}\label{eqn:rwt-bellman-pseudo-labels}
y_{i,h}^{(m\to 0)}
:=
r_{i,h}^{(m)}
+
\gamma\,\widehat\omega_h^{(m)}(s_{i,h+1}^{(m)}\mid s_{i,h}^{(m)},a_{i,h}^{(m)})
\, V_{h+1}^{(0)}(s_{i,h+1}^{(m)}).
\end{equation}
By construction, these pseudo-labels are Bellman-consistent with the target task up to the one-step reward difference identified in Section \ref{sec:RWT-Aligned-Bellman}.
Pooling aligned pseudo-labels across source tasks yields a {\it baseline estimator}
$$
Q^{\text{base}}_{n,h}
\;\approx\;
\mathcal B_h^{(0)} V^{(0)}_{n,h+1}
-
\Delta^{(m)}_{r,h},
$$
which leverages source data to reduce estimation variance without introducing Bellman bias.

\textbf{Stage II. (Target-based correction).} 
Using target-task samples, we form residual labels
\begin{equation}\label{eqn:source-pseudo-labels-residuals}
y_{i,h}^{(0)} = r_{i,h}^{(0)}+\gamma V_{h+1}^{(0)}(s_{i,h+1}^{(0)}),\;
z_{i,h}
=
y_{i,h}^{(0)} - Q^{\text{base}}_{n,h}(s_{i,h}^{(0)},a_{i,h}^{(0)}),
\end{equation}
and estimate a correction $\delta_{n,h}$ that captures the structured one-step reward difference. The resulting transfer update is
\begin{equation} \label{eqn:Q_{n,h}^{trans}}
Q_{n,h}^{\text{trans}}
=
Q^{\mathrm{base}}_{n,h}
+
\delta_{n,h}.   
\end{equation}
This two-stage update directly mirrors the Bellman-level decomposition in Section~\ref{sec:RWT-Aligned-Bellman}: {\it source data contributes primarily to variance reduction, while target data is used to learn the structured task shift}.

\subsection{Exploration and Policy Deployment}
\label{sec:transfer-RWT-deploy}

To ensure exploration, we adopt the principle of optimism in the face of uncertainty. After the transfer update in Section \ref{sec:transfer-RWT-Bellman-update}, we construct an optimistic estimate
$$
Q_{n,h}^{\text{ucb}}(s,a)
\;:=\;
Q^{\text{trans}}_{n,h}(s,a)
+
b_{n,h}(s,a),
%\;\rightarrow\; Q_{n+1,h}(s_{n+1,h},a),
$$
where the exploration bonus $b_{n,h}$ accounts for uncertainty from both stages of estimation: the source-based baseline (including finite source data and density-ratio estimation) and the target-based correction. An explicit form of $b_{n,h}$ for the RKHS instantiation is given in Section \ref{sec:algorithm-rkhs}.

We then set $Q_{n+1,h} = Q_{n,h}^{\text{ucb}}$. 
In episode $n+1$, the agent follows the greedy policy:
$$
a_{n+1,h} \in \arg\max_{a\in\mathcal A} Q_{n+1,h}(s_{n+1,h},a),\; h=1,\ldots,H.
$$
While optimism underpins our theoretical analysis, RWT $Q$-learning itself is exploration-agnostic and can be paired with alternative strategies, such as $\varepsilon$-greedy, when upper confidence bounds are difficult to derive (e.g., under neural network function approximation). We adopt this approach in our empirical evaluation with DQN.

\subsection{Algorithm Summary}

Algorithm~\ref{alg:ofu-rwt} summarizes OFU-RWT $Q$-learning, an optimism-based instantiation of RWT $Q$-learning. Learning proceeds episodically on the target task. In each episode, the agent rolls out a policy on the target task, collects a single trajectory, and then performs synchronous backward Bellman updates.

At each stage, RWT-aligned source pseudo-labels are first used to fit a source-based baseline, which leverages source data for variance reduction while remaining Bellman-consistent with the target task. Using target-task data, a correction term is then estimated to capture the structured one-step task shift. The baseline and correction are combined into a transfer estimate, which is augmented with an exploration bonus to form the updated $Q$-function for the next episode.

\begin{algorithm}[tb!]
\caption{OFU-RWT $Q$-Learning (Model-Agnostic)}
\label{alg:ofu-rwt}
\begin{algorithmic}[1]
\STATE {\bfseries Input:} Horizon $H$, discount $\gamma$; target task $\mathcal M^{(0)}$ (online);
source datasets $\{\mathcal D^{(m)}\}_{m=1}^K$ (online or offline);
density-ratio estimators $\{\widehat\omega_h^{(m)}\}$ (or known $\omega_h^{(m)}$);
function classes $\mathcal F_R$ (baseline) and $\mathcal F_\Delta$ (correction).
\STATE {\bfseries Output:} Policies $\{\pi_n\}_{n=1}^N$ and Q-functions $\{Q_{n,h}\}$.

\STATE {\bfseries Initialization:} 
\STATE Set $Q_{1,h}(s,a)\leftarrow H$ for all $h\in[H]$ and all $(s,a)$.
\STATE Set $V_{1,h}(s)\leftarrow \max_a Q_{1,h}(s,a)$.

\FOR{$n=1,2,\dots,N$}
\STATE {\bfseries (A) Rollout on target task.}
\STATE Observe initial state $s_{n,1}$.
\FOR{Each stage $h=1,\cdots, H$} 
\STATE Select 
$a_{n,h}(s)=\arg\max_a Q_{n,h}(s,a)$.
\STATE Execute $a_{n,h}(s)$ in $\mathcal{M}^{(0)}$, and observe reward $r_{n,h}$ and next state $s_{n,h+1}$. 
\ENDFOR
\STATE Store the trajectory \\
$\tau_n=\{(s_{n,h},a_{n,h},r_{n,h},s_{n,h+1})\}_{h=1}^H$.

\STATE {\bfseries (B) Synchronous backward Bellman updates with RWT.}
\STATE Set $V_{n,H+1}\equiv 0$.
\FOR{$h=H,H-1,\dots,1$}
\STATE \underline{\em Stage I (Source-based Bellman alignment):}
\STATE Construct RWT-aligned pseudo-labels $y_h^{(m\to 0)}$ from $\{\mathcal D^{(m)}\}_{m=1}^K$ using $V_{n,h+1}$ via \eqref{eqn:rwt-bellman-pseudo-labels}, and fit a baseline estimator $Q^{\mathrm{base}}_{n,h}\in \mathcal F_R$.

\STATE \underline{\em Stage II (Target-based correction).}
\STATE Using target data $\{\tau_i\}_{i\le n}$ and residual labels via \eqref{eqn:source-pseudo-labels-residuals}, estimate correction $\delta_{n,h}\in\mathcal F_\Delta$.
\STATE Set $Q^{\mathrm{trans}}_{n,h} \leftarrow Q^{\mathrm{base}}_{n,h} + \delta_{n,h}$.

\STATE \underline{\em (Optimistic $Q$-update).}
\STATE Set $Q_{n+1,h} \leftarrow Q^{\mathrm{trans}}_{n,h} + b_{n,h}$.
\STATE Set $V_{n,h}(s)\leftarrow\max_a Q_{n,h}(s,a)$.
\ENDFOR

\STATE {\bfseries (C) Policy update.}
\STATE For each $h\in[H]$, set $V_{n+1,h}(s)\leftarrow\max_a Q_{n+1,h}(s,a)$ and $\pi_{n+1}$ greedy with respect to $Q_{n+1,h}$.
\ENDFOR
\end{algorithmic}
\end{algorithm}
Section~\ref{sec:rkhs-assume} instantiates this procedure under RKHS function approximation and establishes regret guarantees for the optimism-based variant.
Section~\ref{sec:exp} evaluates RWT $Q$-learning under both tabular and neural network implementations, using $\varepsilon$-greedy exploration when explicit confidence bounds are difficult to derive.

\section{RKHS Instantiation and Regret Analysis}\label{sec:rkhs-assume}

This section instantiates RWT $Q$-learning under reproducing kernel Hilbert space (RKHS) function approximation and establish regret guarantees for the optimism-based variant.
Building on the RWT Bellman alignment framework of Section \ref{sec:RWT-Aligned-Bellman} and the algorithmic structure in Section \ref{sec:OFU-RWT-Q-general}, we formalize the structured one-step task shift in RKHS, derive explicit estimators and exploration bonuses, and analyze the resulting regret. The analysis highlights how Bellman alignment localizes statistical complexity to the task shift, yielding regret bounds that depend on the complexity of the shift rather than that of the target MDP.

\subsection{RKHS Modeling of Structured Task Shift}\label{sec:transferable-rkhs}

We now specify an RKHS instantiation that formalizes the structured one-step task shift identified in Section \ref{sec:RWT-Aligned-Bellman}. Throughout, let $\mathcal K$ be an RKHS over $\mathcal S\times\mathcal A$ with kernel $k$ and feature map $\phi$.

We assume that the {\it target Bellman backup} admits an RKHS representation: for each stage $h$,
$$
(\mathcal B^{(0)}_h V_{h+1})(\cdot,\cdot)\in\mathcal K,
$$
with uniformly bounded RKHS norm. This assumption is standard in kernelized value-based RL and places no restriction on policy structure or value-function similarity across tasks \citep{yang2020function,vakili2023kernelized}.

\textbf{Sufficient conditions.} As a concrete sufficient condition, if the target reward $R^{(0)}_h\in\mathcal K$ and the transition kernel admits a bounded linear operator representation in $\mathcal K$, then the Bellman backup belongs to $\mathcal K$ for any bounded continuation value. We emphasize, however, that our analysis does not rely on explicitly modeling transitions; this condition is provided only for intuition.

Following the Bellman alignment framework of Section \ref{sec:RWT-Aligned-Bellman}, {\it transferability is imposed exclusively on the aligned one-step mismatch}, not on value functions or multi-step Bellman differences.

\textbf{Structured one-step reward difference.}
For each stage $h$ and source task $m$, recall the aligned one-step reward difference
$$
\Delta^{(m)}_{r,h}(s,a):=R^{(0)}_h(s,a)-R^{(m)}_h(s,a).
$$
We assume this discrepancy lies in a lower-complexity RKHS. Further details are left in Appendix~\ref{append:technical},
\begin{assumption}[RKHS-Structured Task Shift]\label{assume:rkhs-structured-task-shif}
For each stage $h$ and source task $m$,  there exists an RKHS $\widetilde{\mathcal K}\subseteq\mathcal K$ with feature map $\widetilde\phi$ such that
$$
\Delta^{(m)}_{r,h}\in\widetilde{\mathcal K},\qquad
\|\Delta^{(m)}_{r,h}\|_{\widetilde{\mathcal K}}\le B_\Delta.
$$
\end{assumption}
The inclusion $\widetilde{\mathcal K}\subseteq\mathcal K$ captures the assumption that the {\it task shift is simpler than the ambient problem}, e.g., smoother, lower-rank, or lower-dimensional. Crucially, we {\it do not} assume similarity of value functions or multi-step Bellman differences across tasks, which are policy-dependent and generally unverifiable.

%%%%%%%%%%%%%%%%%%%%%%%%%%%%%%%%%%%%%%%%%%%%%%%%%%
\subsection{OFU-RWT $Q$-Learning under RKHS Approximation} \label{sec:algorithm-rkhs}

We now instantiate OFU-RWT Q-learning under the RKHS assumptions. This subsection makes explicit how Algorithm \ref{alg:ofu-rwt} is realized with kernel ridge regression and how uncertainty from the two learning stages is quantified.

\textbf{Stage I. (Source-based Bellman alignment).} At each episode $n$ and stage $h$, RWT-aligned pseudo-labels constructed via \eqref{eqn:rwt-bellman-pseudo-labels} are used to estimate a {\it source-based baseline} $Q^{\mathrm{base}}_{n,h}\in\mathcal K$. Concretely, we perform kernel ridge regression
\begin{align}
\label{equ:initial}
Q^{\text{base}}_{n,h} 
\arg\min_{f\in\mathcal K}
\sum_{m=1}^M\sum_{i=1}^{n^{(m)}}
\Big(
f(s_{i,h}^{(m)},a_{i,h}^{(m)})
-
y_{i,h}^{(m\to 0)}
\Big)^2
+
\lambda \|f\|_{\mathcal K}^2.
\end{align}
This estimator leverages abundant source data to reduce variance while remaining Bellman-consistent with the target task up to the structured one-step correction.
No structural assumptions are imposed beyond the RKHS representability of the target Bellman backup.

\textbf{Stage II. (Target-based correction).}
Using target-task data collected up to episode $n$, we estimate a {\it correction term} $\delta_{n,h}\in\widetilde{\mathcal K}$ that captures the structured one-step task shift. Given residual labels defined in \eqref{eqn:source-pseudo-labels-residuals}, we solve
\begin{align}
\label{equ:debiasing}
\delta_{n,h}
=
\arg\min_{f\in\widetilde{\mathcal K}}
\sum_{i=1}^{n}
\Big(
f(s_{i,h}^{(0)},a_{i,h}^{(0)})
-
z_{i,h}
\Big)^2
+
\widetilde\lambda \|f\|_{\widetilde{\mathcal K}}^2.
\end{align}
The resulting {\it transfer estimate} 
$
Q^{\mathrm{trans}}_{n,h}
=
Q^{\mathrm{base}}_{n,h}+\delta_{n,h}
$
combines variance reduction from source data with bias correction from target data, reflecting the Bellman-aligned decomposition of Section \ref{sec:RWT-Aligned-Bellman}.

\textbf{Optimistic $Q$-update.}
To support exploration, we augment the transfer estimate with an exploration bonus,
$
Q^{\text{ucb}}_{n,h}(s,a)
=
Q^{\mathrm{trans}}_{n,h}
+
b_{n,h}(s,a).
$
The exploration bonus $b_{n,h}(s,a)$ accounts for uncertainty from both stages: finite source data (including density-ratio estimation) and online learning of the correction. For RKHS regression, it admits the form (Lemma \ref{lemma:main}):
\begin{align}
\label{equ:bonus}
b_{n,h}(s,a) = %|\eta_n(s,a)|\le
& \beta_{n,m} \sqrt{\phi(s,a)^\top (\Lambda_{n,h}^{(m)}+\lambda I)^{-1}\phi(s,a)} \nonumber \\
& +\beta_{n,0} \sqrt{\widetilde\phi(s,a)^\top (\Lambda_{n,h}^{(0)}+\widetilde\lambda I)^{-1}\widetilde\phi(s,a)}
\end{align}
with confidence parameters $\beta_{n,m}$ and $\beta_{n,0}$ independent of $(s,a)$.
Here $\Lambda^{(m)}_{n,h}$ and $\Lambda^{(0)}_{n,h}$ denote the empirical covariance operators induced by the baseline and correction feature maps $\phi$ and $\widetilde\phi$, respectively, using source samples from task $m$ and target samples up to episode $n$.
%Here {$\Lambda_{n,h}^{(0)}=\sum_{i=1}^n \widetilde\phi(s_{i,h}^{(0)},a_{i,h}^{(0)})\widetilde\phi(s_{i,h}^{(0)},a_{i,h}^{(0)})^\top$} and similarly $\Lambda_{n,h}^{(m)}=\sum_{i=1}^{n^{(m)}} \phi(s_{i,h}^{(m)},a_{i,h}^{(m)})\phi(s_{i,h}^{(m)},a_{i,h}^{(m)})^\top$.

%%%%%%%%%%%%%%%%%%%%%%%%%%%%%%%
\subsection{Regret Analysis}

We now state the regret guarantee for the RKHS instantiation of Algorithm \ref{alg:ofu-rwt}.
Proofs are deferred to Appendix \ref{proof:regret}. 

\subsubsection{Complexity Measures and Their Roles}
The regret bound depends on standard RKHS complexity quantities, reflecting the distinct statistical roles of the source-based baseline and the target-based correction.

\begin{definition}[Baseline RKHS complexity (source stage)] \label{definition:rkhs-complexity-source}
Let $T_{\mathcal K}$ denote the integral operator associated with kernel $\mathcal K$. The effective dimension at regularization level $\lambda>0$,  
$$
\mathcal{N}_0(\lambda)=\operatorname{Tr}\left(T_{\mathcal K}(T_{\mathcal K}+\lambda I)^{-1}\right), 
$$
controls the estimation error of kernel ridge regression with i.i.d. samples.
We assume $\mathcal{N}_0(\lambda)\lesssim \lambda^{-2\beta_0}$, which holds for common kernels with polynomial eigenvalue decay.
\end{definition}
The $N_0(\lambda)$ controls estimation error for kernel ridge regression with i.i.d.~source trajectories. This quantity appears when bounding the baseline estimation error terms in Lemma \ref{lemma:main}, where source samples are independent of the target trajectory and standard RKHS concentration applies.

\begin{definition}[Correction RKHS complexity (target stage)] \label{definition:rkhs-complexity-correction}
For the correction RKHS $\widetilde{\mathcal K}$, target data are collected online and are adaptively dependent. Two complexity measures are therefore required.
\vspace{-2ex}
\begin{itemize}
\item \textbf{Maximal information gain:} 
For the correction kernel $\widetilde{\mathcal K}$, define the maximal information gain 
$$
\Gamma_1(N,\widetilde\lambda)=\sup_{G\in\mathcal G_{\widetilde{\mathcal K},N}}\log\det(I+G/\widetilde\lambda),
$$
where $\mathcal G_{\widetilde{\mathcal K},N}$ is the set of all $N\times N$ Gram matrices induced by $\widetilde{\mathcal K}$.
We assume 
$
\Gamma_1(N,\widetilde\lambda)\;\le\;N^{\beta_1}/\widetilde\lambda^2
$
for $\widetilde\lambda\lesssim1$, which holds for common kernels such as squared exponential or Matérn kernels.
%{\red (target data not iid, use in the online setting)}

\item \textbf{Uniform covering number:} We further assume that the $L_\infty$-covering number of $\widetilde{\mathcal K}$ satisfies $\log \mathcal{N}_\infty(\widetilde{\mathcal K},\epsilon)\le \epsilon^{-\alpha_1}$. 
For the induced quadratic-form class $\mathcal{F}_{\widetilde{\mathcal K}}(\epsilon,\widetilde\lambda):=\{\widetilde{\phi}^\top Q \widetilde{\phi}: \|Q\|_2\le 1/\widetilde\lambda\}$, we assume $\log\mathcal N_\infty(\mathcal F_{\widetilde{\mathcal K}},\epsilon,\widetilde\lambda)
\;\le\;(\widetilde\lambda\epsilon)^{-\alpha_1}$. 
%{\red (target data iid, use in the online setting, UCB bonus, self-normalizing bound, concentration, noise depend on s,a)}
\end{itemize}
\end{definition}
These conditions enter the proof when controlling self-normalized martingale terms arising from online target data (Appendix \ref{proof:regret}). As is standard in kernelized RL with optimism, effective dimension alone is insufficient in the online setting \cite{yang2020function,vakili2023kernelized}.

Here we define the Maximal Information Gain as the maximum over all $N$-points Gram matrices $G$, here $G_{\widetilde{\mathcal K},N}$ is the class of all $N$ by $N$ gram matrices with respect to kernel $\widetilde{\mathcal K}$. This is the online analog of effective dimension
\begin{align*}
\Gamma_1(N,\widetilde \lambda)=\sup_{G\in G_{\widetilde{\mathcal K},N}}\log\det(I+G/\widetilde \lambda)
\end{align*}
Assume $\Gamma_1(N,\widetilde\lambda)\le N^{\beta_1}/\widetilde\lambda^2$ for $\widetilde\lambda\lesssim 1$.
This holds for squared exponential kernel or Matern kernel. 
For details, we refer the reader to Theorem 5 in \cite{srinivas2009gaussian}.

\subsubsection{Source Data and Density-Ratio Conditions}

We next formalize conditions governing how effectively source data can be exploited.

\begin{assumption}[Source coverage]\label{assume:source-coverage}
We assume a uniform source coverage condition:
\begin{align*}
|\phi(s,a)^\top (\Lambda^{(m)}_{n,h}+\lambda)^{-1}\phi(s,a)|\le C_{\operatorname{cov}}/n^{(m)}.
\end{align*}
holds for some coverage parameter $C_{\operatorname{cov}}$ and every source $m$ and stage $h$.
\end{assumption}
This ensures that source data reduce variance rather than induce ill-conditioned estimates. The condition is used to bound the baseline confidence radius in Lemma \ref{lemma:main} and in the regret decomposition.

\begin{assumption}[Density-ratio estimation error] \label{assume:density-ratio} 
Let $\widehat\omega_{i,h}$ be the estimated density ratio. 
Define the per-episode averaged ratio error at stage $h$ as
$\;\mathcal E_{n,h}^2 := \frac{1}{\kappa n}\sum_{m=1}^M\sum_{i=1}^{n^{(m)}}(\hat\omega^{(m)}_{i,h}-\omega^{(m)}_{i,h})^2$. 
We assume the cumulative error satisfies
\begin{align*}
\sum_{n=1}^N \sqrt{\mathcal E_{n,h}^2}
\lesssim 
N^{\frac{2\alpha_0+1}{2(\alpha_0+1)}}\kappa^{-\frac{1}{2(\alpha_0+1)}},
\end{align*}
where $N$ the total number of episodes of interaction with the target task and $\kappa$ is the source-target sampling ratio defined in Section \ref{sec:OFU-RWT-Q-general}.
\end{assumption}

This term appears explicitly in the baseline error decomposition (term $T_2^{(\omega)}$ in Lemma \ref{lemma:main}) and captures {\it imperfect Bellman alignment} due to density-ratio estimation. When $\kappa$ is large or ratios are accurate, this contribution becomes lower order. Estimation strategies are discussed in Appendix~\ref{append:density-ratio}.

\subsubsection{Main Regret Theorem}
We are now ready to state the regret bound.

\begin{theorem}[Regret under RKHS Instantiation]\label{thm:transfer-q-rkhs-regret}
Under the baseline and correction RKHS complexity conditions (Definition \ref{definition:rkhs-complexity-source}--\ref{definition:rkhs-complexity-correction}), together with source coverage (Assumption~\ref{assume:source-coverage}) and density-ratio estimation (Assumption~\ref{assume:density-ratio}), with probability at least $1-2/(NH)$, the regret satisfies
\begin{align*}
\operatorname{Regret}(N) &\le H\sqrt{(N^{\beta_1})\Big[N^{\beta_1+1}+N^{\alpha_1+1}\Big]}\\ 
&+ H\sqrt{NC_{\operatorname{cov}}}N^{\frac{\alpha_0}{2(\alpha_0+1)}}\kappa^{-\frac{1}{2(\alpha_0+1)}}
\end{align*}
\end{theorem}

The regret bound decomposes naturally into two terms. 
\begin{itemize}
\item {\bf Target exploration term.} The {\it first} term depends only on the complexity of task-shift RKHS $\widetilde{\mathcal K}$. This reflects the fact that, after Bellman alignment, exploration is required solely to learn the structured one-step task shift.
\item {\bf Source estimation term.} The {\it second} residual uncertainty from the source-based baseline, including coverage and density-ratio error.
This term decreases with the source-to-target ratio $\kappa$ and becomes negligible when source data are sufficiently abundant.
\end{itemize}

As a sanity check, when $\widetilde{\mathcal K}$ is finite-dimensional (e.g., linear MDPs), the bound recovers the standard $O(H\sqrt N)$ regret. 

\subsubsection{Comparison and Regimes}

\underline{\textit{Single-task RL.}} Without transfer, regret is governed by the ambient RKHS $\mathcal K$, leading to regret of order $\widetilde{\mathcal O}\bigl(HN^{1/2+\beta_0}\bigr)$. 

\underline{\textit{Na\"{i}ve pooled learning.}} 
Pooling Bellman updates without alignment induces persistent bias (Section \ref{sec:RWT-Aligned-Bellman}) and admits no comparable regret guarantee unless tasks are identical; see also Section \ref{sec:exp} for empirical evidence.

\underline{\textit{OFU-RWT transfer.}} By contrast, re-weighted targeting removes Bellman bias at one step and isolates uncertainty to the smaller space $\widetilde{\mathcal K}$. 
When $\alpha_1\le\beta_1$ (a mild condition; see \cite{vakili2023kernelized}) and $\kappa$ is large enough for the source term to saturate, the regret simplifies to
$\operatorname{Regret}(N)
=
\widetilde{\mathcal O}\!\left(HN^{1/2+\beta_0}\right),
$
which is strictly smaller than the single-task rate whenever the task shift is simpler than the ambient problem.

Overall, the theorem shows that Bellman alignment does more than reduce constants: it changes the effective statistical complexity of online RL, replacing dependence on the target MDP with dependence on the structured one-step task difference.

\section{Empirical Experiments} \label{sec:exp}
%Experiment references: \cite{oberst2019counterfactual,bennett2024proximal,chai2025deep}
%\input{data.tex}

\begin{figure*}[tb]
\centering
\includegraphics[width=0.3\linewidth]{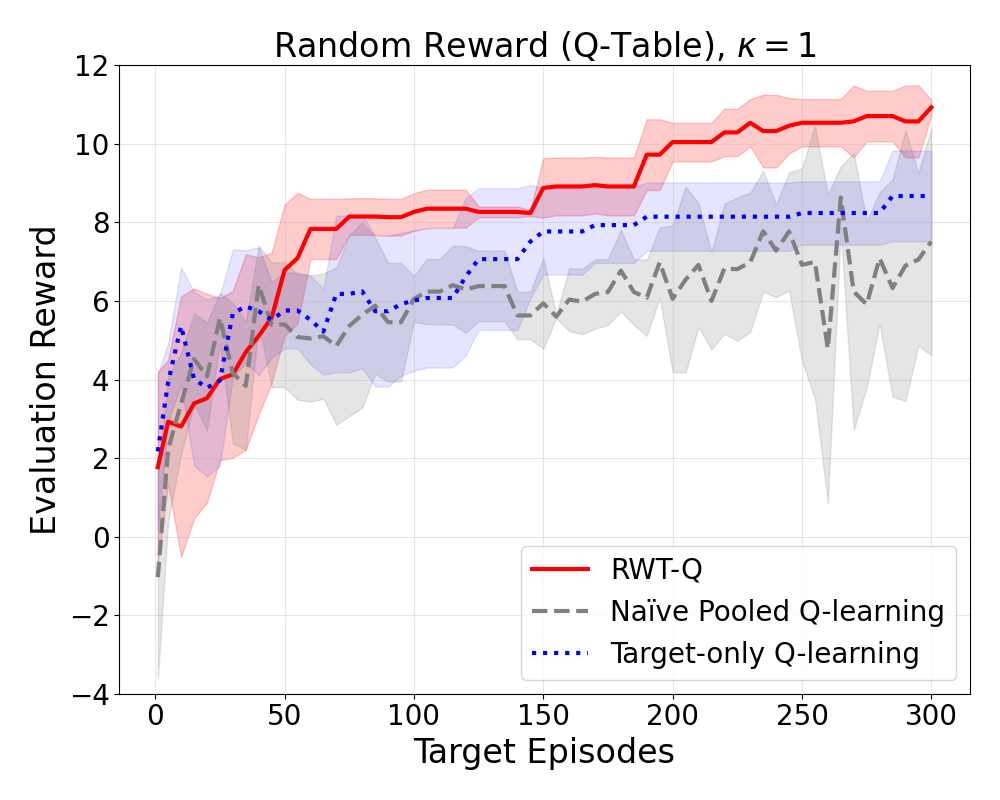}
\includegraphics[width=0.3\linewidth]{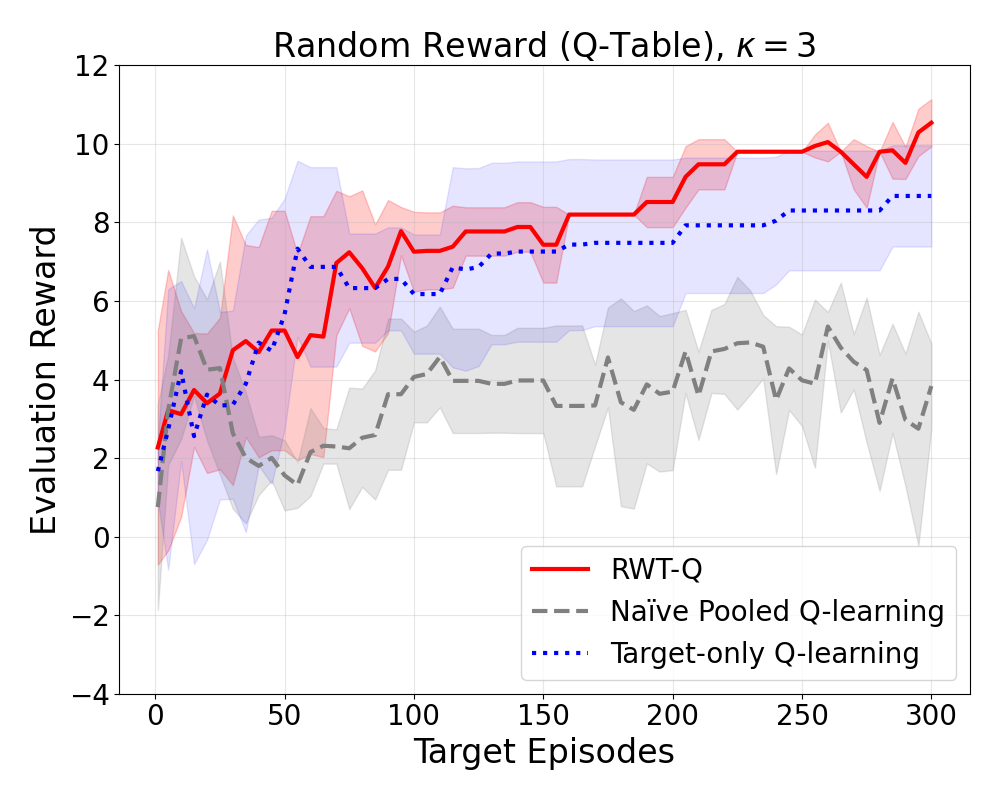}
\includegraphics[width=0.3\linewidth]{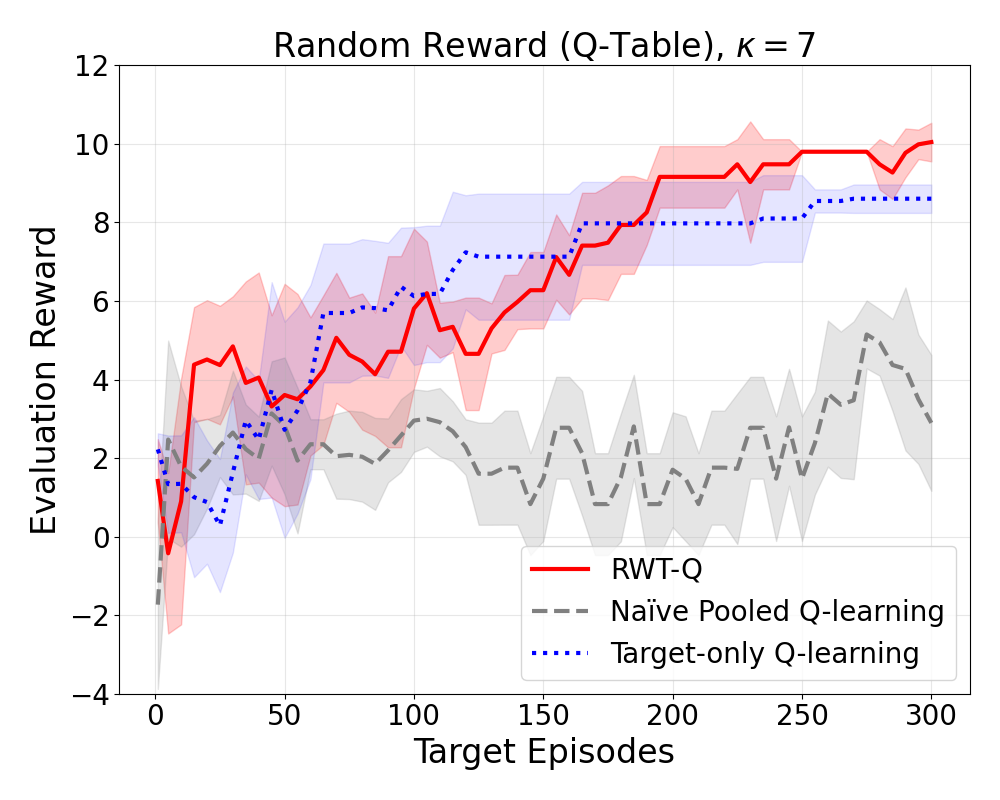}
%%%%%%%%%%%%%%%%%%%%%%%%%%%%%%%%%%%%%%%%%
\includegraphics[width=0.3\linewidth]{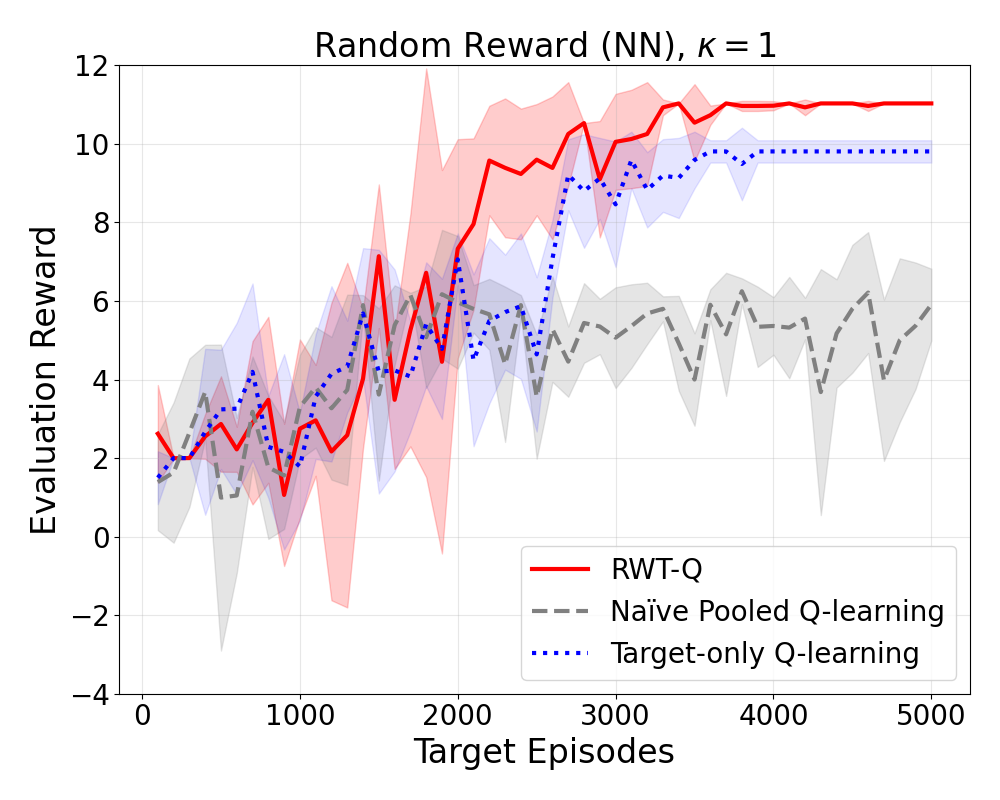}
\includegraphics[width=0.3\linewidth]{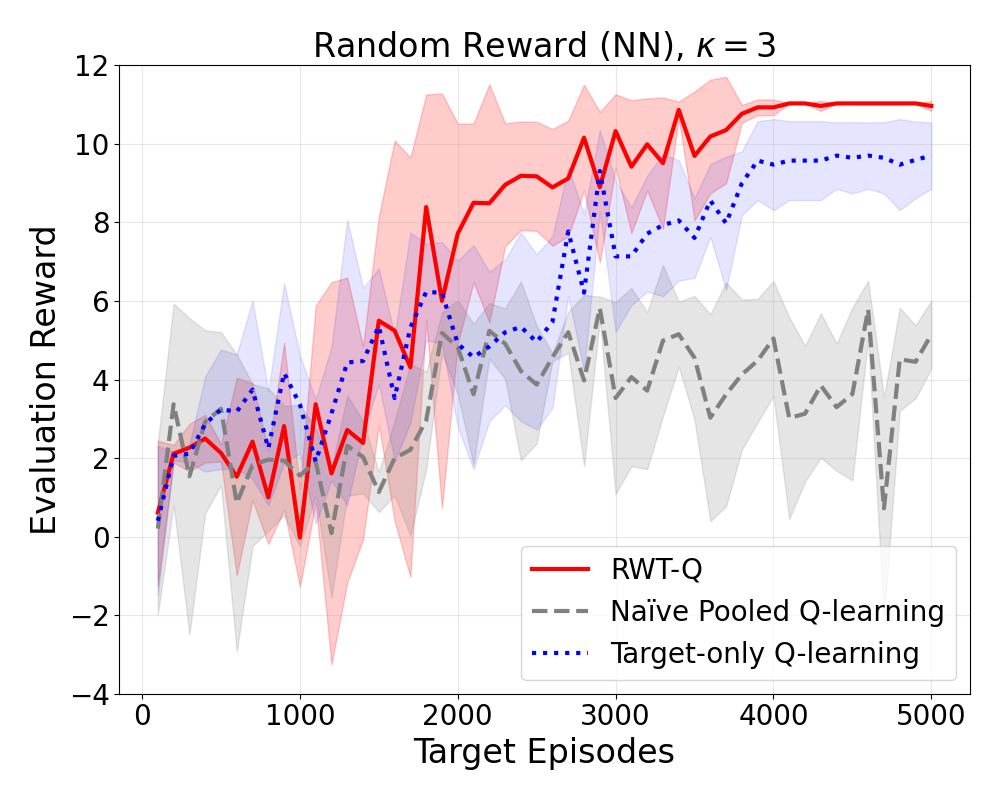}
\includegraphics[width=0.3\linewidth]{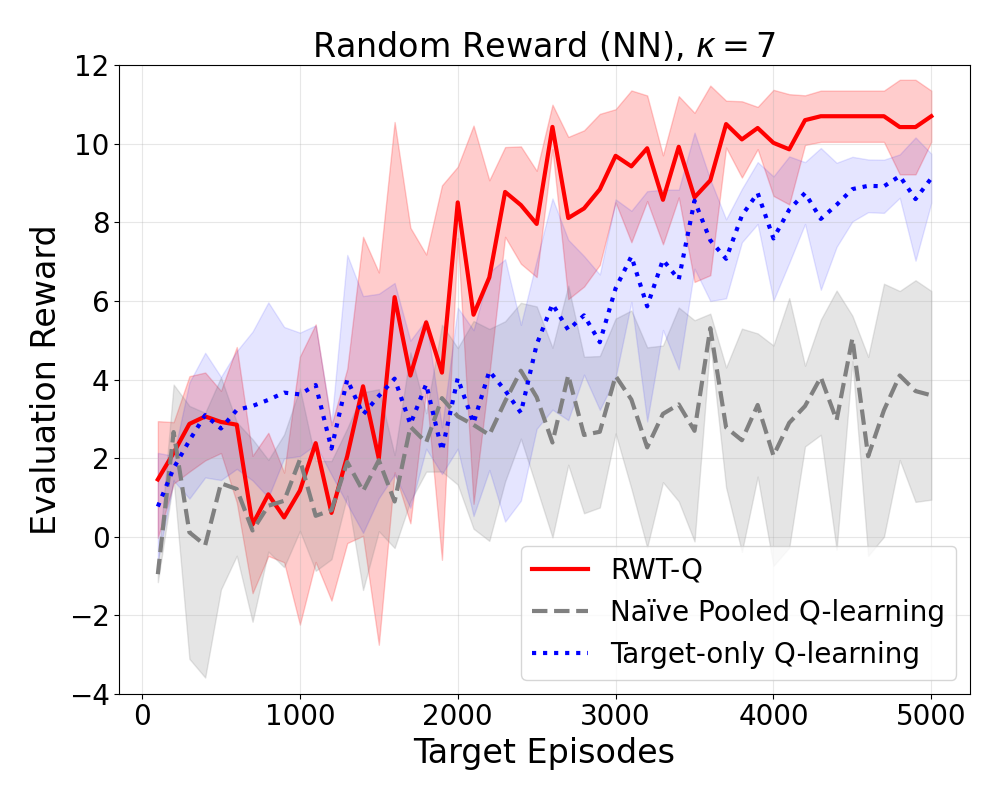}
\caption{Learning curves comparing RWT-$Q$, naïve pooled $Q$-learning, and target-only $Q$-learning. 
{\bf Top}: RandomRewardGridEnv with tabular $Q$-learning. {\bf Bottom}: RandomRewardGridEnv with DQN function approximation. RWT-$Q$ consistently improves sample efficiency, while naïve pooling often degrades performance due to Bellman misalignment.}
\label{fig:exp-res}
\end{figure*}

We evaluate {\it RWT $Q$-learning} on controlled grid-world benchmarks designed to test transfer under structured one-step reward shifts and to expose the failure of naïve Bellman reuse. All methods use identical exploration strategies and differ only in how source data are incorporated.

\subsection{Experimental Setup}

\textbf{RandomRewardGridEnv.} To control task mismatch, we construct environments with synthetic reward surfaces. The target reward is generated as a fixed Q-table with i.i.d. Gaussian entries, while the source reward is defined as
$$
R_{\text{source}}(s,a)=R_{\text{target}}(s,a)+\Delta(s,a),
$$
where $\Delta(s,a) \overset{\text{i.i.d.}}{\sim} \mathcal{N}(0, \sigma_\Delta^2)
$ is a zero-mean Gaussian perturbation, where the standard deviation $\sigma_\Delta$ controls the complexity of the task shift. This environment directly instantiates the structured one-step reward-difference model in Section \ref{sec:RWT-Aligned-Bellman}.

\textbf{Benchmarks.}
We compare the following methods:

\textbf{$\bullet$ RWT-$Q$ (ours).} Implements the two-stage Bellman-aligned update: a source-based baseline Q-function combined with a target-only correction term. This corresponds exactly to the RWT decomposition analyzed in Sections \ref{sec:OFU-RWT-Q-general}-\ref{sec:rkhs-assume}.

\textbf{$\bullet$ Na\"{i}ve pooled $Q$-learning.} Trains a single Q-function on pooled source and target data, implicitly assuming tasks are identical. This baseline reflects common heuristic transfer approaches and directly violates Bellman alignment.

\textbf{$\bullet$ Target-only $Q$-learning.} Standard $Q$-learning trained solely on target-task data, ignoring source information.

All methods use identical exploration strategies ($\epsilon$-greedy) and differ only in how source data are incorporated.

\textbf{Implementations.} 
We run experiments with both {\it tabular $Q$-learning} and {\it neural network} function approximation. 
For neural experiments, all methods share the same architecture. 
% We vary the source-target data ratio $\kappa$ and the shift magnitude $\sigma_\Delta$.
 We vary the source-target data ratio $\kappa$.
 Results are averaged over multiple random seeds. See more implementation details in \cref{appsec:exp}.

\subsection{Results} 

Figure \ref{fig:exp-res} reports undiscounted episode return averaged over seeds.
Across all setting, {\it RWT-Q consistently outperforms both baselines}. In particular:

$\bullet$ Na\"{i}ve pooling often degrades performance relative to target-only learning, confirming the structural Bellman bias predicted in Section 2.

$\bullet$ RWT-$Q$ reliably improves sample efficiency, especially when the reward shift is structured but nontrivial.

$\bullet$ Performance gains persist under neural network approximation, indicating that Bellman alignment provides benefits beyond the tabular setting.

Overall, these results empirically validate the central theoretical message of the paper: {\it transfer in online RL should respect Bellman alignment}, and na\"{i}ve reuse of Bellman updates can be harmful even when tasks are closely related.

\section{Conclusion}

We showed that na\"{i}ve transfer in online reinforcement learning fails due to continuation-value-dependent Bellman bias, and identified one-step Bellman alignment as the correct abstraction for principled transfer. We introduced re-weighted targeting (RWT), an operator-level correction that removes this dependence and reduces task mismatch to a fixed one-step correction, yielding a two-stage learning framework that separates variance reduction from bias correction.

Instantiated as RWT $Q$-learning and analyzed via an optimism-based variant, our theory establishes regret bounds that scale with the complexity of the task shift rather than the target MDP. Empirical results with both tabular and neural network implementations further demonstrate that Bellman alignment provides consistent benefits beyond the specific settings covered by our analysis.

% ---------- Bibliography ----------
\bibliographystyle{plainnat} % or plain / alpha; change if you prefer
\bibliography{main}

% ---------- Appendix ----------
%\newpage
%\appendix
%% Keep appendix double-spaced as well (default). If you want single-spaced appendix, wrap in \begingroup\singlespacing ... \endgroup.
%%\onecolumn % not needed for generic TR; enable if you want ICML-like appendix layout
%\input{9-rkhs.tex}
%\input{9-proof.tex}
%\input{9-more.tex}

\end{document}